%% file: main.tex
\def\year{2018}\relax
\pdfoutput=1
\documentclass[letterpaper]{article} 
\usepackage{aaai18}  
\usepackage{times}  
\usepackage{helvet}  
\usepackage{courier}  
\usepackage{url}  
\usepackage{graphicx}  
\usepackage{booktabs}
\usepackage{placeins}
\usepackage{float}
\usepackage{amsmath}
\usepackage{graphicx}
\usepackage{amssymb}
\usepackage{comment,mdwlist}
\usepackage{adjustbox}

\usepackage{amsthm}

\usepackage{url}

\usepackage{multirow,color}
\usepackage{array, xspace}
\usepackage{multirow}
\usepackage{multicol}
\usepackage{enumitem}

\theoremstyle{definition}
\newtheorem{definition}{Definition}

\theoremstyle{theorem}
\newtheorem{theorem}{Theorem}

\def\eqref#1{(\ref{#1})}

\usepackage[noend]{algorithmic}  
\usepackage{algorithm}

\algsetup{indent=2em,linenodelimiter=.,linenosize=\small} 

\newcommand{\femb}[1]{\ensuremath{\mathbf{d}_{#1}}}
\newcommand{\feat}[2]{\ensuremath{\mathbf{f}_{#1#2}}}
\newcommand{\emb}[1]{\ensuremath{\mathbf{x}_{#1}}}
\newcommand{\con}[1]{\ensuremath{\mathbf{y}_{#1}}}
\newcommand{\seine}{sign2vec}
\newcommand{\sine}{SiNE}
\newcommand{\sam}[1]{\ensuremath{\mathbf{\eta}_{#1}}}
\newcommand{\samp}[1]{\ensuremath{\mathbf{\eta}_{#1}^{+}}}
\newcommand{\samn}[1]{\ensuremath{\mathbf{\eta}_{#1}^{-}}}

\begin{document}
%
\title{Distributed Representations of Signed Networks}
\author{Mohammad Raihanul Islam,
B. Aditya Prakash, Naren Ramakrishnan\\
Discovery Analytics Center, Department of Computer Science, Virginia Tech\\
Email: \texttt{\{raihan8, badityap, naren\}@cs.vt.edu}
}
\maketitle

\input{sections/abstract.tex}
\input{sections/intro.tex}
\input{sections/problem.tex}

\input{sections/methods.tex}

\input{sections/exp.tex}
\input{sections/reference.tex}
\input{sections/conc.tex}


\input{ref.bbl}
\bibliographystyle{aaai}
\end{document}

%% file: sections/abstract.tex
\begin{abstract}
Recent successes in word embedding and document embedding have motivated researchers to explore similar representations for networks and to use such representations for tasks such as edge prediction, node label prediction, and community detection. Existing methods are largely focused on finding distributed representations for unsigned networks and are unable to discover embeddings that respect polarities inherent in edges. We propose \seine{}, a fast scalable embedding method suitable for signed networks. Our proposed objective function aims to carefully model the social structure implicit in signed networks by reinforcing the principles of social balance theory. Our method builds upon the traditional word2vec family of embedding approaches but we propose a new targeted node sampling strategy to maintain structural balance in higher-order neighborhoods. We demonstrate the superiority of \seine{} over state-of-the-art methods proposed for both signed and unsigned networks on several real world datasets from different domains. In particular, \seine{} offers an approach to generate a richer vocabulary of features of signed networks to support representation and reasoning.
\end{abstract}

%% file: sections/intro.tex
\section{Introduction}\label{sec:intro}
Social and information networks are ubiquitous today across a variety
of domains; as a result, a large
body of research has been developed to help construct
discriminative and informative features for network analysis tasks such as
classification~\cite{node}, prediction~\cite{link}, visualization~\cite{tsne}, 
and entity recommendation~\cite{recom}.



Classical approaches to find features and embeddings are
motivated by dimensionality reduction research and extensions, e.g.,
approaches such as 
Laplacian eigenmaps~\cite{lle}, non-linear 
dimension reduction~\cite{isomap,nonlin}, and spectral embedding~\cite{specG,specE}. 
More recent research has focused on developing network analogies to
distributed vector representations such as word2vec~\cite{w2v1,w2v2}.
In particular, by viewing sequences of nodes encountered on random walks as
documents, methods such as 
DeepWalk~\cite{deepwalk}, node2vec~\cite{node2vec}, and LINE~\cite{line}
learn similar representations for nodes (viewing them as words).

Although these approaches are scalable to large networks, they are
primarily applicable to only unsigned networks. 
Signed networks are becoming increasingly important in
online media, trust management, and in law/criminal applications.
As we will show, applying
the above methods to signed networks results in key information loss in the
resulting embedding. For instance, if the sign between two nodes is negative,
the resulting embeddings could place the nodes in close proximity, which
is undesirable.

A recent attempt to fill this gap is the work of~\citeauthor{sine}
wherein the authors learn node representations by optimizing an 
objective function through a multi-layer neutral network based on structural 
balance theory. This work, however, models only local connectivity information
through 2-hop paths and fails to capture
global balance structures prevalent in a network.
Our contributions are:
\begin{enumerate}[leftmargin=0cm,itemindent=0.5cm,labelwidth=\itemindent,labelsep=0cm,align=left,nosep]
\item We propose \seine{}, a scalable node embedding method for feature learning in signed networks that maintains structural balance in higher 
order neighborhoods. \seine{} is very generic by design, and
can handle both directed and undirected networks,
including weighted or unweighted (binary) edges.
\item We propose a novel node sampling method as an improvement over traditional negative sampling. The idea is
to keep a cache of nodes during optimization integral for maintaining the principles of structural balance in the network. This targeted node sampling can be treated as an extension of the negative sampling used in word2vec models.
\item Through extensive experimentation, we demonstrate that \seine{} generates better features suitable for a range of prediction tasks such as edge and node label prediction. \seine{} is able to scalably
generate embeddings for networks with millions of nodes.
\end{enumerate}

%% file: sections/problem.tex
\vspace{-1mm}
\section{Problem Formulation}\label{sec:prob}
\vspace{-0.5mm}
\begin{definition}{\textit{Signed Network:}}
A signed network can be defined as $G=(V,E)$, where $V$ is the set of vertices and $E$ is the set of edges between the vertices. Each element $v_i$ of $V$ represents an entity in the network and each edge $e_{ij}\in E$ is a tuple ($v_i$, $v_j$) associated with a weight $w_{ij}\in \mathbb{Z}$. The absolute value of $w_{ij}$ represents the strength of the
relationship between $v_i$ and $v_j$, whereas the sign represents the nature of relationship (e.g., friendship or antagonism). A signed network can be either directed or undirected. If $G$ is undirected then the order of vertices is not relevant (i.e. $(v_i,v_j)\equiv(v_j,v_i)$). On the other hand, if $G$ is directed then order becomes relevant (i.e. $(v_i,v_j)\not\equiv(v_j,v_i)$ and $w_{ij}\neq w_{ji}$)).
\end{definition}\vspace{-1mm}

\begin{figure}[!hbtp]
\centering
\includegraphics[scale=0.4]{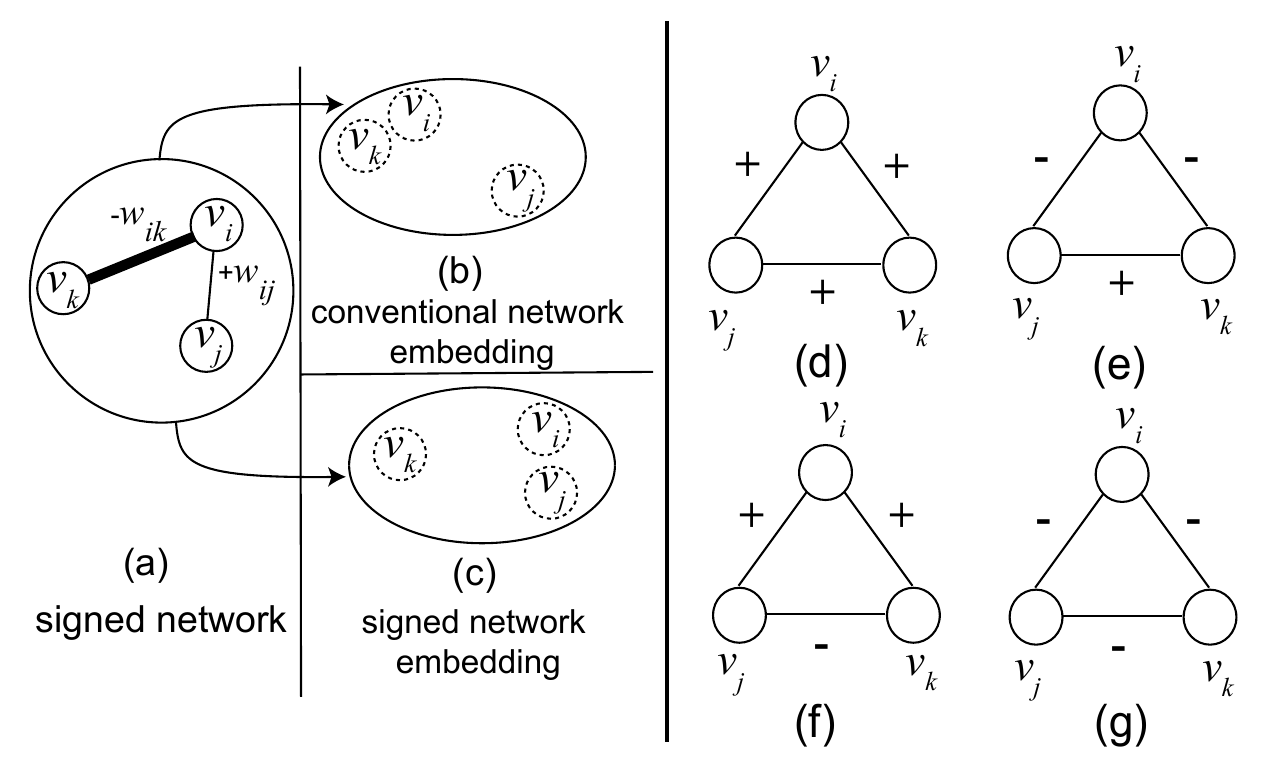}
\caption{\textit{Given a signed network (a), a conventional network embedding (b)
does not take signs into account and can result in faulty representations.
(c) \seine{} learns embeddings that respect sign information between edges. Of the possible signed triangles, (d) and (e) are considered balanced but (f) and (g) are not.}}
\label{fig:embed}
\vspace{-1mm}
\end{figure}
Because the weights in a signed network carry a combined
interpretation (sign denotes polarity and magnitude denotes
strength),
conventional proximity assumptions used in unsigned network
representations (e.g., in~\cite{node2vec}) cannot
be applied for signed networks. 
Consider a network wherein
the nodes $v_i$ and $v_j$ are positively connected and the
nodes $v_k$ and $v_i$ are negatively connected (see Fig.~\ref{fig:embed}(a)). 
Suppose the weights of the edges $e_{ij}$ and $e_{ik}$ are $+w_{ij}$ and $-w_{ik}$ respectively. Now if $|+w_{ij}| < |-w_{ik}|$, conventional embedding methods will place $v_i$ and $v_k$ closer than $v_i$ and $v_j$ owing to the
stronger influence of the weight (Fig. \ref{fig:embed}(b)). Even if considering the weight of negative edge as zero does not resolve it, because even though it may put node $v_i$ and $v_j$ closer, node $v_k$ may be relatively closer to $v_i$ because of ignoring the adverse relation between node $v_i$ and $v_k$. This may comprise the quality of embedding space. Ideally, we would like a representation
wherein nodes $v_i$ and $v_j$ are closer than nodes $v_i$ and $v_k$, as shown in Fig. \ref{fig:embed}(c). 
This example shows that modeling the polarity is as important as modeling the strength of the relationship.


To accurately model the interplay between the vertices in signed networks we use the theory of
\textit{structural balance} proposed by~\citeauthor{stbal1}. Structural balance theory posits that triangles with an odd number of positive edges are more plausible than an even positive edges (see Fig.~\ref{fig:embed}). 
Although different adaptation and alternative of balance theory exist in the literature, here we focus primarily on the original notion of structural balance to create the embedding space because it is useful in many scenarios like signed networks constructed from adjectives (described in Section \ref{sec:exp}).

\textbf{Problem Statement: }{\textit{Scalable Embedding of Signed Networks (\seine{}):}}
Given a signed network $G$, compute a low-dimensional vector $\femb{i}\in \mathbb{R}^K,\,\forall v_i\in V$, where positively related vertices reside in close proximity and negatively related vertices are distant.

%% file: sections/methods.tex
\section{Scalable Embedding of Signed Networks (\seine{})}\label{sec:des}
\subsection{\seine{} for undirected networks}
Consider a weighted signed network defined as in section \ref{sec:prob}. 
Now suppose each $v_i$ is represented by a vector $\mathbf{x}_i\in \mathbb{R}^K$. Then a natural way to compute the proximity between $v_i$ and $v_j$ is by the following function (ignoring the sign for now):
\begin{equation}\label{eqn:p1}
p_{u}(v_i,v_j) = \sigma(\emb{j}^T \cdot \emb{i}) = \dfrac{1}{1 + \exp(-\emb{j}^T\cdot \emb{i})}
\end{equation} 
where $\sigma(a) = \frac{1}{1 + \exp(-a)}$. 
Now let us breakdown the weight of edge $w_{ij}$ into two components:
$r_{ij}$ and $s_{ij}$. $r_{ij}\in \mathbb{N}$ represents the absolute value of $w_{ij}$ (i.e. $r_{ij}=|w_{ij}|$) and $s_{ij}\in \{ -1,1\}$ represents the sign of $w_{ij}$. Given this breakdown of $w_{ij}$, $p_{u}(v_i,v_j) = \sigma(s_{ij} (\emb{j}^T\cdot \emb{i}))$. Now  incorporating the weight information, the objective function for undirected signed network can be written as:
\begin{equation}\label{eqn:obj1}
\mathcal{O}_{un} = \sum_{e_{ij}\in E} r_{ij} \sigma(s_{ij} (\emb{j}^T\cdot \emb{i}))= \sum_{e_{ij}\in E}r_{ij} p_{u}(v_i,v_j)
\end{equation}
By maximizing Eqn. \ref{eqn:obj1} we obtain a vector $\emb{i}$ of dimension $K$ for each node $v_i\in V$ (we also use
$\femb{i}$ to refer this embedding, for reasons that will become
clear in the next section).
\subsection{\seine{} for directed networks}
Computing embeddings for directed networks is trickier due to the
asymmetric nature of neighborhoods (and thus, contexts). For instance,
if the edge $e_{ij}$ is positive, but $e_{ji}$ is negative, it is not clear
if the respective representations for nodes $v_i$ and $v_j$ should be 
proximal or not. We solve this problem by treating each vertex as itself plus a specific context; for instance,
a positive edge $e_{ij}$ is interpreted to mean
that given the context of node $v_j$, node $v_i$ wants to be closer. This enables
us to treat all nodes consistently without worrying about reciprocity
relationships. To this end, we introduce
another vector $\con{i}\in\mathbb{R}^K$ besides $\emb{i}$, $\forall v_i\in V$. For a directed edge $e_{ij}$ the probability of context $v_j$ given $v_i$ is:
\begin{equation}\label{eqn:p2}
p_{d}(v_j|v_i)=\dfrac{\exp (s_{ij} (\con{j}^T\cdot \emb{i}))}{\sum_{k=1}^{|V|} \exp( s_{ik}(\con{k}^T \cdot \emb{i}))}
\end{equation}
Treating the same entity as itself and as a specific context is very popular in the text representation literature~\cite{w2v1}. The above equation defines a probability distribution over all context space w.r.t. node $v_i$. Now our goal is to optimize the above objective function for all the edges in the network. However we also need to consider the weight of each edge in the optimization. Incorporating the absolute weight of each edge we obtain the objective function for a directed network as:
\begin{equation}\label{eqn:obj2}
 \mathcal{O}_{dir} = \sum_{e_{ij}\in E} r_{ij}p_{d}(v_j|v_i)
 \end{equation} 
By maximizing Eqn.~\ref{eqn:obj2} we will obtain two vectors $\emb{i}$ and $\con{i}$ for each $v_i\in V$. The vector $\emb{i}$ models the outward connection of a node whereas $\con{i}$ models the inward connection of the node. Therefore the concatenation of $\emb{i}$ and $\con{i}$ represents the final embedding for each node. We denote the final embedding of node $v_i$ as $\femb{i}$. It should be noted that for undirected network $\femb{i}=\emb{i}$ whereas for a
directed network $\femb{i}$ is the concatenation of $\emb{i}$ and $\con{i}$. This means $|\emb{i}|=|\con{i}|=\frac{K}{2}$ in the case of directed graph (for the same representational length).
\subsection{Efficient Optimization by Targeted Node Sampling}\label{sec:samp}
The denominator of Eqn. \ref{eqn:p2} is very hard to compute as we have to marginalize the conditional probability over the 
entire vertex set $V$. We adopt 
the classical negative
sampling approach~\cite{w2v3} wherein
negative examples are selected from some distribution for 
each edge $e_{ij}$. 
However, for signed network conventional negative sampling does not work. For example consider the network from Fig.~\ref{fig:negs}(a). 
Viewing this example
as an unsigned network, while optimizing for edge $e_{ij}$,
we will consider $v_i$ and $v_y$ as negative examples and thus they will
be placed distantly from each other. However, in a signed network
context, $v_i$ and $v_y$ have a friendlier relationship (than with, say, $v_x$)
and thus should be placed closer to each other. 
We propose a new sampling approach, referred to as simply
\textit{targeted node sampling} wherein we
first create a cache of nodes for each node with their estimated relationship according to structural  balance theory and then sample nodes accordingly.

\begin{figure}[!t]
  \centering
  \begin{tabular}{@{}c@{}c@{}}
  \includegraphics[scale=0.3]{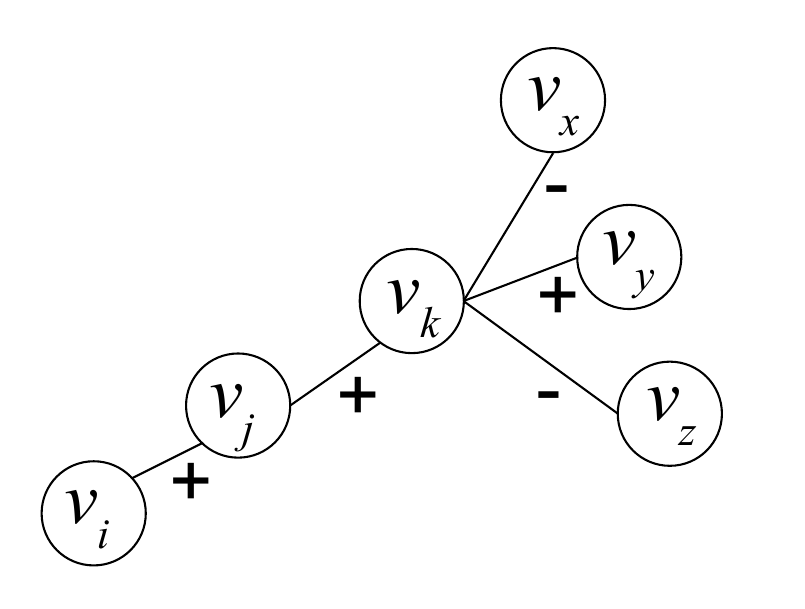} &
  \includegraphics[scale=0.25]{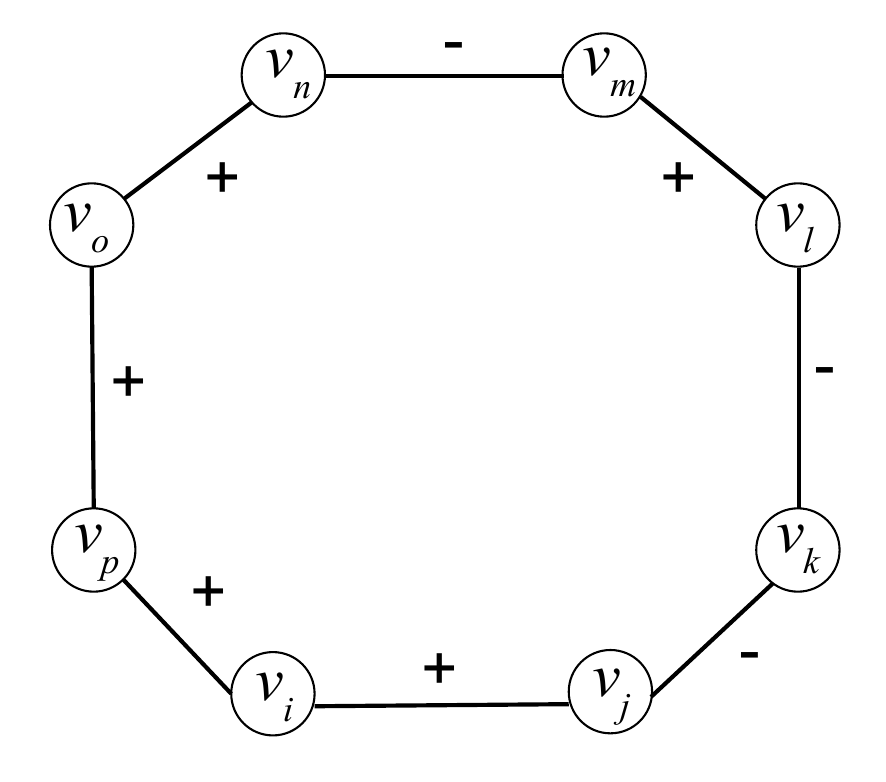} \\
  (a) & (b)
  \end{tabular}
  \caption{(\textit{a) depicts a small network to illustrate
why conventional negative 
sampling does not work. $v_i$ and $v_y$ might be considered
too distant for their representations to be placed close to each other.
Targeted node sampling solves this problem
by constructing a cache of nodes which can be used as sampling. (b) shows how we resolve \textit{conflict}. Although there are two ways to proceed 
from node $v_i$ to $v_l$ the shortest path is $v_i,v_j,v_k,v_l$, which estimates a net
positive relation between $v_i$ and $v_l$. As a result $v_l$ will be added to $\samp{i}$. However for node $v_m$ there are two shortest paths from $v_i$, with the path $v_i,v_p,v_o,v_n,v_m$ having more positive edges but
with a net negative relation,
so $v_m$ will be added to $\samn{i}$ in case of a \textit{conflict}.}}
\label{fig:negs}
\end{figure} 
\subsubsection{Constructing the cache for each node:}\label{sec:cache}
We aim to construct a cache of positive and negative examples for each node $v_i$ where the positive (negative) example cache $\samp{i}$ ($\samn{i}$) contains nodes which  should have a positive (negative) relationship with $v_i$ according to structural balance theory. 
To construct these caches for each node $v_i$, we apply random walks of length $l$ starting with $v_i$ to obtain a sequence of nodes. Suppose the sequence is $\Omega=<v_i, v_{n_0}, \cdots, v_{n_{l-1}}>$. Now we add each node $v_{n_p}$ to either $\samp{i}$ or $\samn{i}$ by observing the estimated sign between $v_i$ and $v_{n_p}$. The estimated sign is computed using the following recursive formula: 
\begin{equation}\label{eqn:sign}
\tilde{s}_{i{n_p}} = \tilde{s}_{in_{p-1}} \times s_{n_{p-1}n_p}
\end{equation} 
Here $\tilde{s}_{in_{p-1}}$ is the estimated sign between node $v_i$ and node $v_{n_{p-1}}$, which can be computed recursively. The base case for this formula is $\tilde{s}_{i{n_1}} = s_{i{n_0}}\times s_{{n_0}{n_1}}$. 
If node $v_{n_p}$ is not a neighbor of node $v_i$ and $\tilde{s}_{i{n_p}}$ is positive then we add $v_{n_p}$ to $\samp{i}$. On the other hand if $\tilde{s}_{i{n_p}}$ is negative and $v_{n_p}$ is not a neighbor of $v_i$ then we add it to $\samn{i}$.  For example for the graph shown in Fig. \ref{fig:negs}(a), suppose a random walk starting with node $v_i$ is $<v_i, v_j, v_k, v_z>$. Here node $v_k$ will be added to $\samp{i}$ because $\tilde{s}_{ik}=s_{ij}\times s_{jk}>0$ (base case) and $v_k$ is not a neighbor of $v_i$. However, $v_z$ will be added to node $\samn{i}$ since $\tilde{s}_{iz}=\tilde{s}_{ik}\times s_{kz}<0$ and $v_z$ is not a neighbor of $v_i$. 

The one problem with this approach is that 
a node $v_j$ may be added to both $\samp{i}$ and $\samn{i}$. We denote this phenomena as \textit{conflict} and define the reason for this \textit{conflict} in Theorem~\ref{th:un}. We resolve this situation
by computing the shortest path between $v_i$ and $v_j$ and compute $\tilde{s}_{ij}$ between them using the shortest path, then add to either $\samp{i}$ or $\samn{i}$ based on $\tilde{s}_{ij}$. To compute the shortest path we have to consider the network as unsigned since negative weight has a different interpretation for shortest path algorithms. We also prove that if there are multiple shortest paths with equal length in case of a \textit{conflict}, then only one path has the highest number of positive edges. We pick this path to compute $\tilde{s}_{ij}$. Both proofs are described in the supplementary section. A scenario is shown in Fig.~\ref{fig:negs}(b). 
\begin{theorem}\label{th:un}
(Reason of conflict): Node $v_j$ will be added to both $\samp{i}$ and $\samn{i}$ if there are multiple paths from $v_i$ to $v_j$ and the union of these paths has at least one unbalanced cycle.  
\end{theorem}
\begin{proof}
(By contradiction.) Suppose there is a \textit{conflict} for node $v_i$ where $\samp{i}$ and $\samn{i}$ both contain node $v_j$. Since there are at least two distinct $v_i$-$v_j$ paths because of the conflict, the network contains a cycle $c$ (ignoring the direction for directed networks). Now it is evident that the common edges of both paths are not responsible for the conflict since they occur in both paths. Suppose the cycle has two distinct $x$-$y$ paths. Now if cycle $c$ is balanced there will be an even number of negative edges which will be distributed between the distinct  $v_x$-$v_y$ paths in $c$. The distribution can occur in two ways: either both paths will have an odd number of negative edges or an even number of negative edges. In both cases the estimated sign between the $v_x$-$v_y$ paths will be the same. However, this is a contradiction because the final estimated sign of two $v_i$-$v_j$ paths are different and the signs between the common path are same, so thesigns between the $v_x$-$v_y$ paths must be different. Therefore, cycle $c$ cannot be balanced and hence contains an odd number of negative edges. Thus we have identified at least one unbalanced cycle.
\end{proof}
\subsubsection{Targeted edge sampling during optimization:}\label{subsec:samp}
Now after constructing the cache $\sam{i}=\samp{i}\bigcup \samn{i}$ for each node $v_i$, we can apply the targeted sampling approach
for each node. Here our  goal is to extend the objective of negative sampling from
classical word2vec approaches~\cite{w2v3}. In traditional negative sampling, a
random word-context pair is negatively sampled for each observed word-context pair. In a signed network both positive and negative edges are present, and thus  we aim to conduct both types of
sampling while sampling an edge observing its sign. Therefore when
sampling a positive (negative) edge $e_{ij}$, we aim to sample multiple 
negative (positive) nodes from $\samn{i}$ ($\samp{i}$). Therefore the objective function for each edge becomes (taking $\log$):
\begin{equation}\label{eqn:objf}
\scriptsize
\mathcal{O}_{ij} = \log[\sigma(s_{ij} (\con{j}^T \cdot \emb{i}))] + \sum_{c=1}^\mathcal{N}E_{v_n \sim \tau(s_{ij})} \log[\sigma(\tilde{s}_{in} (\con{n}^T\cdot \emb{i}))]
\end{equation}
Here $\mathcal{N}$ is the number of targeted node examples per edge and $\tau$ is a function which selects from $\samp{i}$ or $\samn{i}$ based on the sign $s_{ij}$. $\tau$ selects from $\samp{i}$ ($\samn{i}$) if $s_{ij}<0$ ($s_{ij}>0$). 

\begin{figure}[!t]
  \centering
  \begin{tabular}{@{}c@{}c@{}}
  \includegraphics[width=0.25\textwidth]{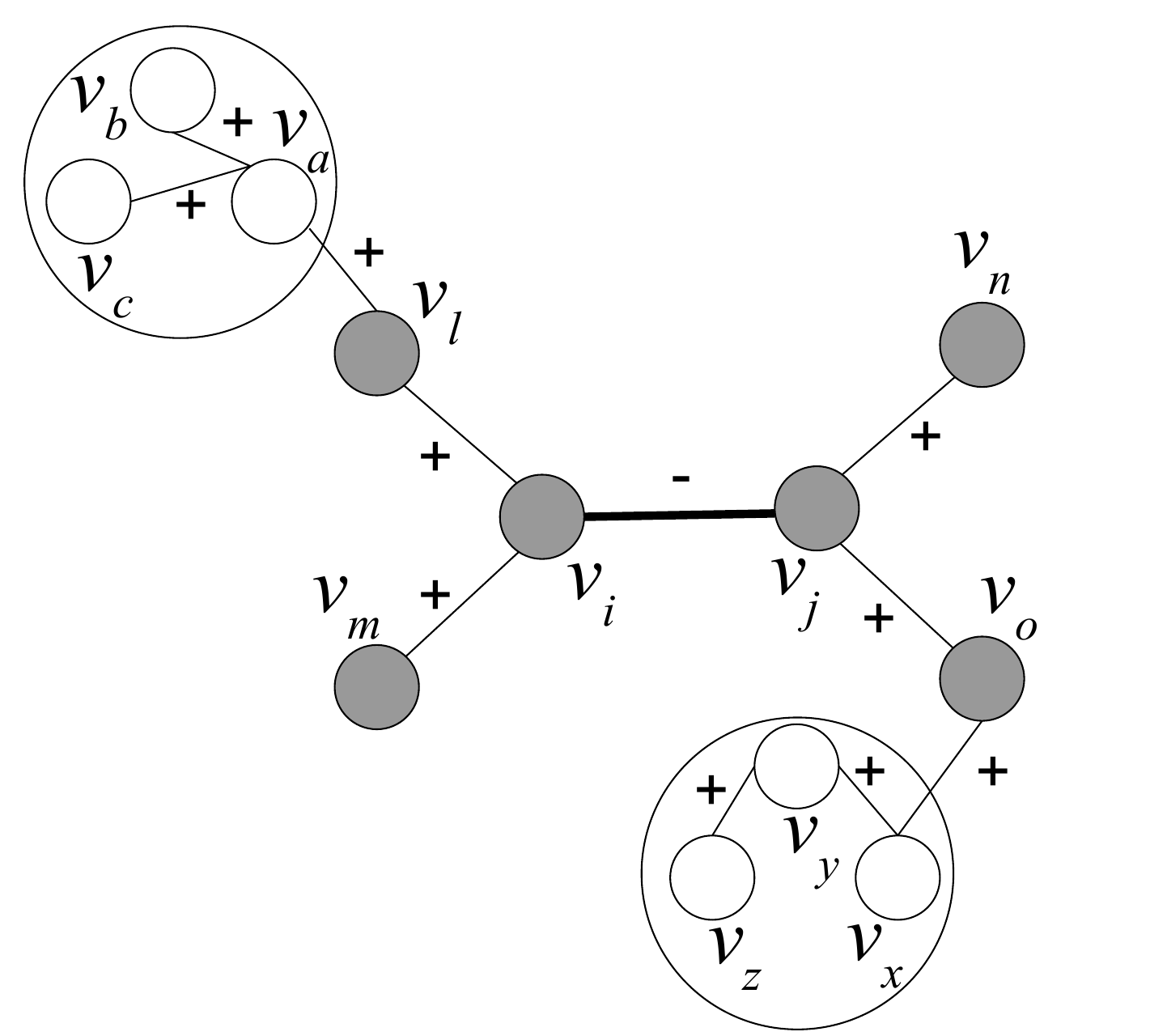} & 
  \includegraphics[width=0.25\textwidth]{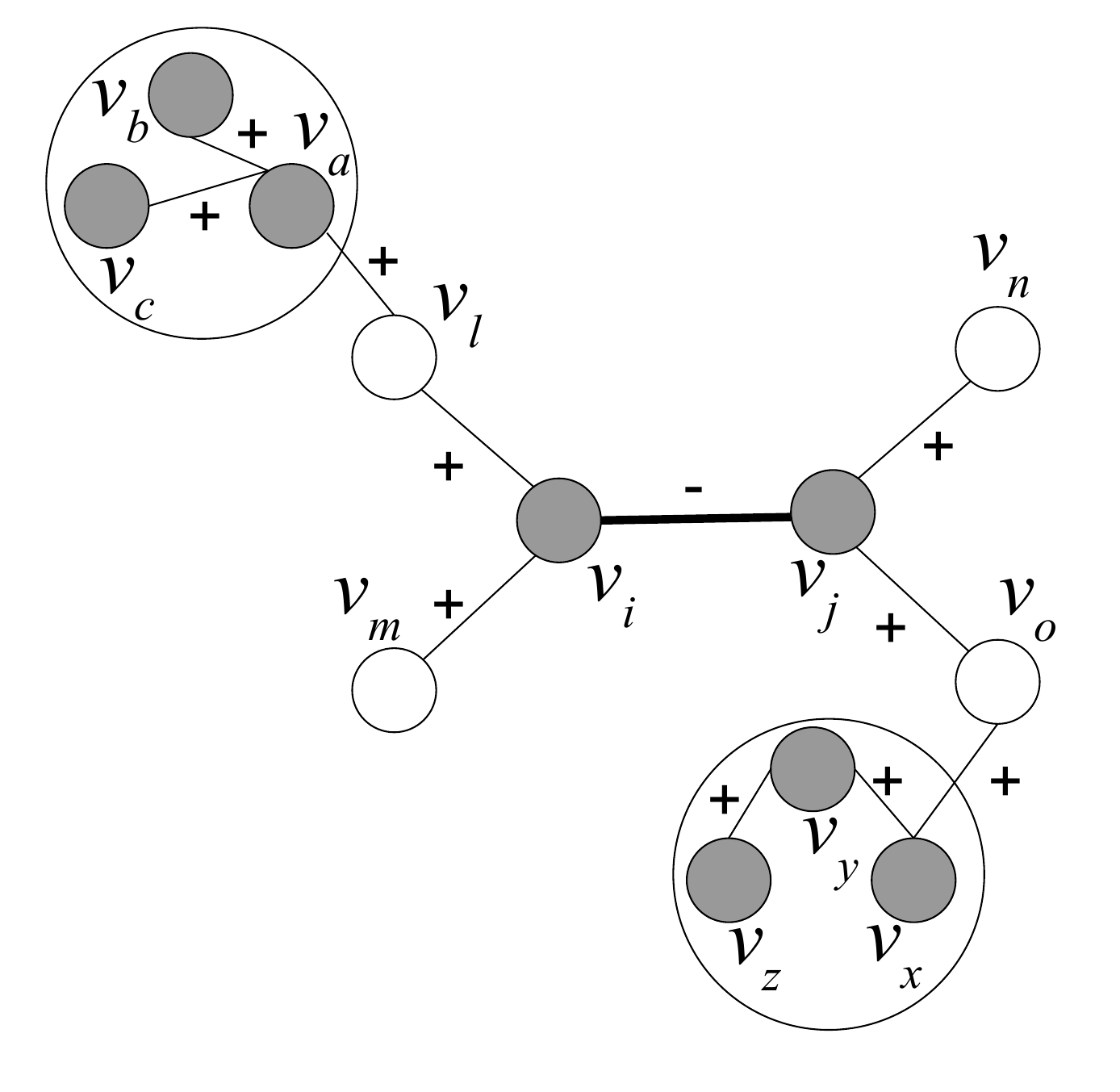} \\
  (a) \sine{} & (b) \seine{}
  \end{tabular}
  \caption{\textit{A comparative scenario depicting the optimization process inherent
in both SiNE (a) and \seine{} (b). The shaded vertices represent the nodes both methods will consider while optimizing the edge $e_{ij}$. We can see the \sine{} only considers the immediate neighbors because it optimizes edges in $2$-hop paths having opposite signs. On the other hand, \seine{} considers higher order neighbors ($v_a,v_b,v_c,v_x,v_y,v_z$) for targeted node sampling.}}
 \label{fig:comp_node}
\end{figure}

The benefit of targeted node sampling in terms of global balance considerations
across the entire network  is shown in Fig.~\ref{fig:comp_node}. Here we compare how our proposed approach
\seine{} and \sine{}~\cite{sine} maintain structural balance. For simplicity suppose only edge $e_{ij}$ has a negative sign. Now \sine{} only optimizes w.r.t.
pairs of edges in $2$-hop paths each having different signs. Therefore optimizing the edge $e_{ij}$ involves only the immediate neighbors of node $v_i$ and $v_j$, 
i.e. $v_l,v_m,v_n,v_o$ (Fig. \ref{fig:comp_node} (a)). However \seine{} skips the immediate neighbors while it uses higher order neighbors (i.e., $v_a,v_b,v_c,v_x,v_y,v_z$). Note
that \seine{} actually uses immediate neighbors as separate examples (i.e edge $e_{il}, e_{im}$ etc.). In this manner \seine{} covers more nodes to optimize the embedding space than \sine{}. 
\begin{algorithm}
\caption{The \seine{} algorithm}
\label{alg:main}
\begin{algorithmic}[1]
\REQUIRE (Graph $G=(V,E)$, embedding size $K$, walks per node $r$, walk length $l$, total number of samples $s$, initial learning rate $\gamma$)
\ENSURE $\femb{k}\in \mathbb{R}^K, \forall v_k\in V$
\FORALL{$v_n\in V$}
	\FOR{$i=1$ to $r$}
		\STATE $\omega_{ni}=RandomWalk(G,v_n,l)$
	\ENDFOR 
\ENDFOR
\FORALL{$v_n\in V$}
	\FOR{$i=1$ \TO $r$}
		\FOR{each $v_k\in \omega_{ni}$}
			\STATE Estimate relation between $v_k$ and $v_n$ using Eqn. \ref{eqn:sign}
			\STATE Add $v_k$ to either $\samp{n}$ or $\samn{n}$ based on the relation
		\ENDFOR
	\ENDFOR
	\STATE resolve \textit{conflict} for node $v_n$
\ENDFOR
\REPEAT
\FOR{each mini-batch of edges}
	\STATE Sample an edge using edge sampling method 
	\STATE Optimize the objective function in Eqn. \ref{eqn:objf}.
\ENDFOR
\STATE Update learning rate $\gamma$
\UNTIL{in total $s$ samples are processed}
\end{algorithmic}
\end{algorithm}
\subsection{Discussion}\label{sec:dis}
We now discuss several computational aspects of the \seine{} model.

\textbf{Optimization:} We adopt the asynchronous stochastic gradient method (ASGD) \cite{hogwild} to optimize the objective function $\mathcal{O}_{ij}$ for each edge $e_{ij}$. The ASGD method randomly selects a mini batch of randomly selected edges and update emebeddings at each step. Now for each edge $e_{ij}$ the gradient of the objective function will have a constant coefficient $r_{ij}$ (i.e. $|w_{ij}|$) . Now if the absolute weights of the edges have a high variance, it is hard to find a good learning rate. For example if we set the learning rate very small it would work well for large weighted edge but for small weighted edge the overall learning will be very inadequate resulting in poor performance. On the other hand, a large learning rate will work well for edges with smaller weights but for edges with large weight the gradient will be out of limits. To remedy this we adopt the \textit{edge sampling} used in \cite{line}. In edge sampling all the weighted edges treated as binary edges with non-negative weights (i.e. absolute value of edges $r_{ij}$). Now the edges are sampled during optimization according to the multinomial distribution constructed from the absolute value of the edge weights. For example suppose all the absolute values of the edges are stored in the set $R=\{ r_1, r_2, \cdots r_{|E|} \}$. Now during the optimization each edge is sampled according to the multinomial distribution constructed from $R$. However, each sampling from $R$ would take $O(E)$ time, which is computationally expensive for large network. To remedy this we use the alias table approach proposed in~\cite{alias}. An alias table takes $O(1)$ time while continuously drawing samples from a constant discrete multinomial distribution. 

\textbf{Threshold value for $\sam{i}$:} Theoretically there should not be any bound on the size of $\samp{i}$ and $\samn{i}$. However empirical analysis shows limiting the size of $\samp{i}$ to very small values (i.e $5-7$) actually gives better results. 

\textbf{$\sam{i}$ for low degree nodes:} Nodes with a low degree may not have an adequate number of samples for $\samp{i}$ and $\samn{i}$ from the random walks. This is why it is possible to exchange the nodes within $\samp{i}$ and $\samn{i}$. For example if node $v_x\in \samp{i}$, one can add node $v_i$ to $\samp{x}$. 

\textbf{Embedding for new vertices:} \seine{} can learn embedding for newly arriving vertices. Since this is a network model, we can assume that advent of new vertices means we know its connection with existing nodes (i.e., neighbors). Suppose the new vertex is $v_n$ and its set of neighbors is $\mathbf{N}_n$. We just have to construct $\sam{n}$ and optimize the newly formed edges using the same optimization function stated in Eqn.~\ref{eqn:objf} to obtain the embedding of node $n$.

\textbf{Complexity:} Constructing $\sam{i}$ for node $v_i$ takes $O(rl)$ time where $l$ is the length of random walk and $r$ is the number of walk for each node. Since $rl\ll |V|$, the total cache construction actually takes very little time w.r.t. vertex size. Moreover conflict resolution only takes place for very rare instances where the length of the shortest path is at most $l$. This cost is thus negligible compared to random walk and cache construction time. Now, for optimizing each edge along with the node sampling take $O(K(\mathcal{N}+1))$, where $K$ is the size of embedding space and $\mathcal{N}$ is the size of node sampling. The total complexity of optimization then become $O(K(\mathcal{N}+1)|E|)$, where $E$ is the set of edges. Therefore the overall complexity becomes $O(rl|V| + K(\mathcal{N}+1)|E|)$. A pseudocode of \seine{} is shown in Algorithm \ref{alg:main}. \seine{} is available at: https://github.com/raihan2108/signet.

%% file: sections/exp.tex
\section{Experiments}\label{sec:exp}
\subsubsection{Experimental Setup:} \label{sec:setup}
We compare our algorithm against both the state-of-the-art method proposed for signed and unsigned network embedding. The description of the methods are below:
\begin{itemize}[leftmargin=0cm,itemindent=0.5cm,labelwidth=\itemindent,labelsep=0cm,align=left,nosep]
\item node2vec~\cite{node2vec}: This method, not
specific to signed networks, computes embeddings
by optimizing the neighborhood structure using informed random walks.
\item SNE~\cite{sne}: This method computes the embedding using a log bilinear model; however it does not exploit any specific theory of signed networks.
\item SiNE~\cite{sine}: This method uses a multi-layer neural network 
to learn the embedding by optimizing an objective function satisfying 
structural balance theory. \sine{} only concentrates on the immediate 
neighborhood of vertices rather than on the global balance structure. 
\item \seine{}-NS: This method is similar to our proposed method \seine{} except it uses conventional negative sampling instead of our proposed targeted node sampling.
\item \seine{}: This is our proposed \seine{} method which uses random walks to construct a cache of positive and negative examples for targeted node sampling.
\end{itemize}
We skip hand crafted feature generation method for link prediction like \cite{posneg} because they can not be applied in node label prediction and already shows inferior performance compared to SiNE.

In the discussion below, we focus on five real world signed network datasets (see Table \ref{tab:stat}). Out of these five, two datasets are from 
social network platforms---Epinions and Slashdot---courtesy the
Stanford Network Analysis Project (SNAP). The details on how the signed 
edges are defined are available at the project website~\footnote{http://snap.stanford.edu/}. The third dataset is a voting records of Wikipedia adminship election (Wiki), also from SNAP. 
The fourth dataset we study is an adjective network (ADJNet) constructed from the synonyms and antonyms collected from Wordnet database. Label information about whether the adjective is positive or negative comes from SentiWordNet~\footnote{http://sentiwordnet.isti.cnr.it/}. The last dataset  is a citation network we constructed from written case opinions of the Supreme Court of the United States (SCOTUS). We expand the notion of SCOTUS citation network~\cite{scotus} into a signed network. 

To understand this network, it is important to note that
there are typically two main parts to a SCOTUS case opinion. The first part contains the majority and any optional concurring opinions where justices cite previously argued cases to defend their position. The second part (optional, does not exist in a unanimous decision) consists of dissenting opinions containing arguments opposing the decision of the majority opinion. In our modeling, nodes denote cases (not opinions). The citation of one case's majority opinion to another case will form a positive relationship, and citations from dissenting opinions will form a negative relationship. We collected all written options from the inception of SCOTUS to construct the citation network. Moreover, we also collected the decision direction of supreme court cases from The Supreme Court  Database~\footnote{http://scdb.wustl.edu/}. This decision direction denotes whether the decision is conservative or liberal, information that we will use for validation. We also use 3 synthetic datasets in \ref{sec:node}, details are in the corresponding section.

Unless otherwise stated, for directed networks we set $|\emb{i}|=|\con{i}|=\frac{K}{2}=20$ for both \seine{}-NS and \seine{}; therefore $|\femb{i}|=40$. For a fair comparison, the final embedding dimension for others methods is set to $40$. For undirected network (ADJNet) $|\femb{i}|=40$ for all the methods. We also set the total number of samples (examples) to 100 million, $\mathcal{N}=5$, $l=50$ and $r=1$ for \seine{}-NS and \seine{}. For all the other parameters for node2vec, SNE and \sine{} we use the settings recommended in their respective papers.

\begin{table}[!t]
\centering
\scriptsize
\caption{\textit{Statistics of the datasets used for performance evaluation. In social network datasets negative edges are underrepresented, however in ADJNet and SCOTUS they are well represented. ADJNet and SCOTUS also contain binary labels.}}
\label{tab:stat}
\begin{tabular}{l|l|l|l|l|l}
\toprule
Statistics        & Epinions & Slashdot & Wiki & ADJNet & SCOTUS   \\ \midrule
total nodes       & 131828   & 82144 & 7220  & 4579   & 28305     \\
positive edges    & 717667   & 425072 & 83717 & 10708  & 43781   \\
negative edges    & 123705   & 124130 & 28422 & 7044   & 42102    \\
total edges       & 841372   & 549202 & 112139 & 17752  & 85883   \\
\% negative edges & 14.703   & 22.602 & 25.345 & 39.680 & 49.023 \\ 
direction & directed & directed & directed & undirected & directed \\ \bottomrule
\end{tabular}
\end{table}
\begin{figure}[!hbtp]
\centering
\includegraphics[scale=0.07]{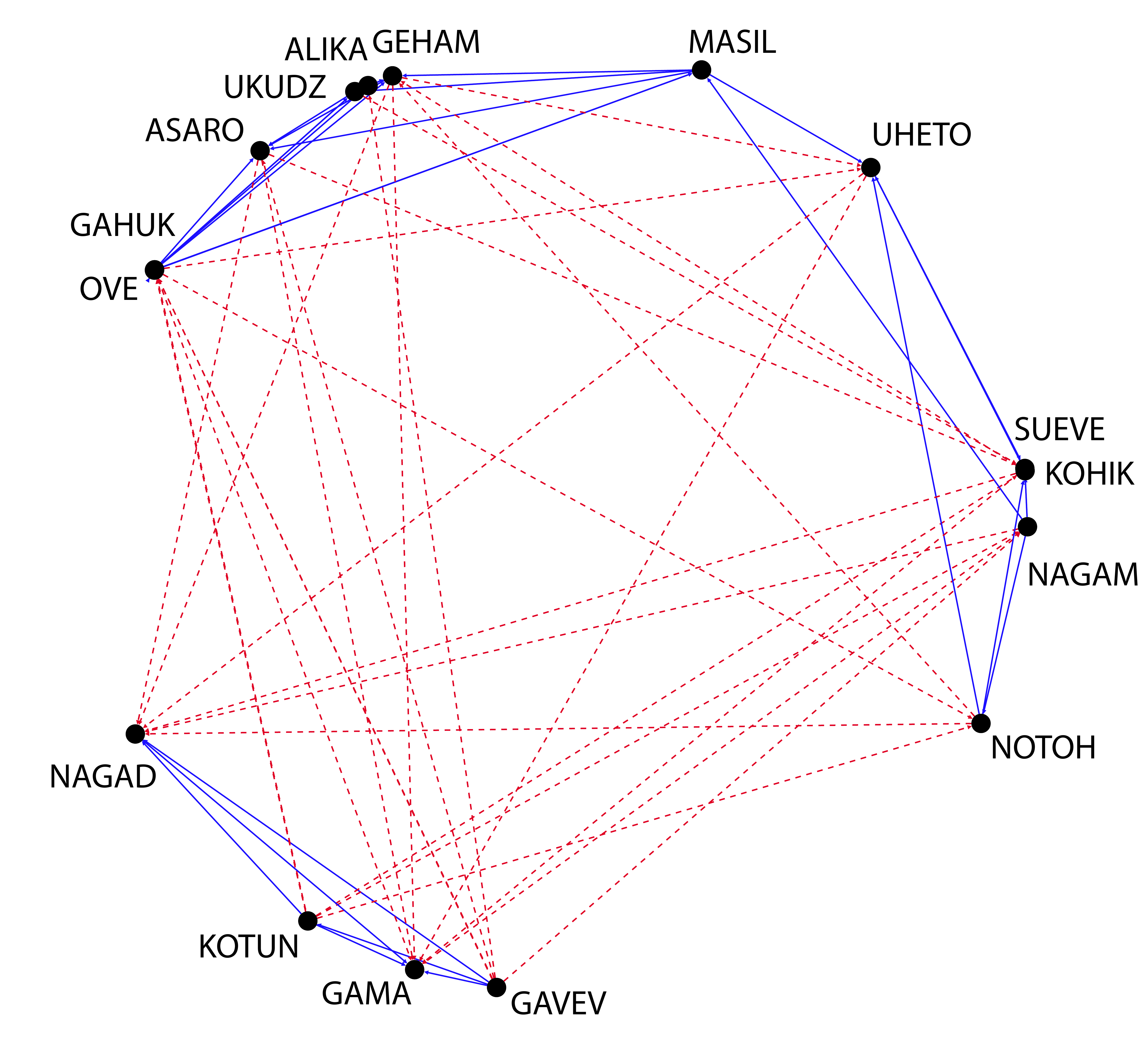}
\caption{\textit{2-dimensional embedding of alliances among sixteen tribes of New Guinea. Alliance (hostility) between the tribes is shown in solid blue (dashed red) edges. We can see that edges representing alliance are comparatively shorter than the edges represents hostility.}}
\label{fig:tribe}
\end{figure}
\subsubsection{Are Embeddings Interpretable?}\label{sec:embed-int}
For visual depiction of embeddings,
we first utilize a small dataset denoting relations between sixteen tribes in
Central Highlands of New Guinea~\cite{tribe}. This is a signed network showing the alliance and hostility between the tribes. We learned the embeddings in two dimensional space as an undirected network as shown in Fig. \ref{fig:tribe}. We can see that in general solid blue edges (alliance) are shorter than the dashed red edges (hostility) confirming that allied tribes are closer than the hostile tribes. One notable point is tribe \textit{MASIL} has no enemies and often works as a peace negotiator between the tribes. We can see that
\textit{MASIL} positions nicely between two groups of tribes \{\textit{OVE, GAHUK, ASARO, UKUDZ, ALIKA, GEHAM}\} and \{\textit{UHETO, SEUVE, NAGAM, KOHIK, NOTOH}\}. The tribes within these two groups are only allied to each other and \textit{MASIL} but they are hostile to other tribes belonging to different groups. This actually justifies the position of \textit{MASIL}. As reported in \cite{antho} there is another such group which consists of the tribes \textit{NAGAD, KOTUN, GAMA, GAVEV}; notice that they position themselves in the lower left corner far away from other two groups. Therefore the embedding space learned by \seine{} clearly depicts alliances and relationships among the tribes.

\subsubsection{Does the embedding space learned by \seine{} support structural balance theory?}\label{sec:ana}
Here we present our analysis on whether the embedding space learned by \seine{} follows the principles of structural balance theory. 
We calculate the mean Euclidean distance between
representations of nodes connected by positive versus negative
edges, as well as their standard deviations 
(see Table~\ref{tab:dist}). The lower value of positive edges suggests positively connected nodes stay closer together than the negatively connected nodes indicating that \seine{} has successfully learned the embedding using the principles of structural balance theory. Moreover, the ratio of average distance between the positive and negative edges is at most $67\%$ over all the datasets suggesting that \seine{} grasps the principles very effectively. 

\begin{table}[!t]
\centering
\scriptsize
\caption{\textit{Average Euclidean distance between node representations connected by
positive edges versus negative edges with std. deviation. We can see that the avg. distance between positive edge is significantly lower than negative edges indicating that \seine{} preserves the conditions of structural balance theory.}}
\label{tab:dist}
\begin{tabular}{@{}l|l|l|l|l|l}
\toprule
\begin{tabular}[c]{@{}l@{}}Type of \\ edges\end{tabular} & Epinions    & Slashdot    & Wiki        & SCOTUS       & ADJNet      \\ \midrule
positive                                                 & 0.86 (0.37) & 0.98 (0.31) & 1.06 (0.27) & 0.84 (0.25) & 0.71 (0.16) \\
negative                                                 & 1.64 (0.23) & 1.60 (0.19) & 1.56 (0.19) & 1.64 (0.21) & 1.77 (0.08) \\ \midrule
ratio                                                    & 0.524       & 0.613       & 0.679       & 0.512        & 0.401       \\ \bottomrule
\end{tabular}
\end{table}
\subsubsection{Are representations learned by \seine{} effective at edge label prediction?}\label{sec:edge}
We now explore the
utility of \seine{} for edge label
prediction. For all the datasets we sample $50\%$ of the edges as a
training set to learn the node embedding. Then we train a logistic regression classifier using the embedding as features and the sign of the edges as label. This classifier is used to predict the sign of the remaining $50\%$ of the edges. Since edges involve two nodes we explore several scores to compute the features for edges from the node embedding. They are described below:
\begin{enumerate}[leftmargin=0cm,itemindent=0.5cm,labelwidth=\itemindent,labelsep=0cm,align=left,nosep]
\item \textbf{Concatenation} (concat): $\feat{ij}=\femb{i} \oplus \femb{j}$
\item \textbf{Average} (avg): $\feat{ij}=\frac{\femb{i} + \femb{j}}{2}$
\item \textbf{Hadamard} (had): $\feat{ij}=\femb{i} * \femb{j}$
\item \textbf{$\mathcal{L}_1$}: $\feat{ij}=|\femb{i}-\femb{j}|$
\item \textbf{$\mathcal{L}_2$}: $\feat{ij}=|\femb{i}-\femb{j}|^2$
\end{enumerate}

\begin{table}[!t]
\scriptsize 
\caption{\textit{Comparison of edge label prediction in all datasets. We show the macro F1 score for each feature scoring method. 
The best score across all the scoring method is shown in boldface. \seine{} outperforms node2vec, SNE, and \sine{} in every case. The results are statistically significant with $\mathit{p}<0.01$.}}
\centering
\label{tab:edge}
\begin{tabular}{p{0.6cm}|p{1.5cm}|>{\centering\arraybackslash}p{0.7cm}|>{\centering\arraybackslash}p{0.7cm}|>{\centering\arraybackslash}p{0.5cm}|>{\centering\arraybackslash}p{0.7cm}|>{\centering\arraybackslash}p{0.7cm}}
\hline
Eval.                      & Dataset    & Epinions & Slashdot & Wiki  & ADJNet & SCOTUS \\ \hline
\multirow{5}{*}{concat}    & node2vec        & 0.601    & 0.508    & 0.45  & 0.478  & 0.500  \\
                           & SNE             & 0.461    & 0.436    & 0.428 & 0.376  & 0.447  \\
                           & SiNE            & 0.460    & 0.436    & 0.427 & 0.401  & 0.378  \\
                           & \seine{}-NS   & 0.792    & 0.654    & 0.719 & 0.379  & 0.547  \\
                           & \seine{}      & \textbf{0.807}   & \textbf{0.716}   & \textbf{0.750} & 0.412  & 0.550  \\ \hline
\multirow{5}{*}{avg}       & node2vec        & 0.485    & 0.495   & 0.428 & 0.477  & 0.495  \\
                           & SNE             & 0.461    & 0.436    & 0.428 & 0.376  & 0.363  \\
                           & SiNE            & 0.460    & 0.436    & 0.427 & 0.388  & 0.378  \\
                           & \seine{}-NS   & 0.626    & 0.589    & 0.614 & 0.374  & 0.509  \\
                           & \seine{}      & 0.694    & 0.668    & 0.667 & 0.400  & 0.523  \\ \hline
\multirow{5}{*}{had}       & node2vec        & 0.469    & 0.455    & 0.428 & 0.43   & 0.492  \\
                           & SNE             & 0.461    & 0.436    & 0.428 & 0.376  & 0.336  \\
                           & SiNE            & 0.460    & 0.436    & 0.427 & 0.393  & 0.378  \\
                           & \seine{}-NS   & 0.666    & 0.554    & 0.508 & \textbf{0.795}  & 0.671  \\
                           & \seine{}      & 0.726    & 0.582    & 0.523 & 0.785  &  \textbf{0.815}  \\ \hline
\multirow{5}{*}{$\mathcal{L}_1$}        & node2vec        & 0.461    & 0.437    & 0.431 & 0.401  & 0.492  \\
                           & SNE             & 0.461    & 0.436    & 0.428 & 0.376  & 0.378  \\
                           & SiNE            & 0.460    & 0.436    & 0.427 & 0.378  & 0.378  \\
                           & \seine{}-NS   & 0.661    & 0.552    & 0.457 & 0.792  & 0.598  \\
                           & \seine{}      & 0.753    & 0.627    & 0.487 & 0.788  & 0.782  \\ \hline
\multirow{5}{*}{$\mathcal{L}_1$}        & node2vec        & 0.464    & 0.439    & 0.432 & 0.451  & 0.483  \\
                           & SNE             & 0.461    & 0.436    & 0.428 & 0.376  & 0.336  \\
                           & SiNE            & 0.460    & 0.436    & 0.427 & 0.378  & 0.378  \\
                           & \seine{}-NS   & 0.665    & 0.560    & 0.463 & \textbf{0.795} & 0.630  \\
                           & \seine{}      & 0.760    & 0.641    & 0.508 & 0.786  & 0.792  \\ \hline
\multicolumn{2}{l|}{gain over node2vec (\%)}      & 34.28    & 40.94    & 66.67 & 64.85  & 63.00     \\
\multicolumn{2}{l|}{gain over SNE (\%)}           & 75.05    & 64.22    & 75.23 & 109.57 & 82.33  \\
\multicolumn{2}{l|}{gain over SiNE (\%)}          & 75.43    & 64.22    & 75.64 & 96.51   & 115.61 \\
\multicolumn{2}{l|}{gain over \seine{}-NS (\%)} & 1.89     & 9.48     & 4.31  & -0.88  & 21.46  \\ \hline
\end{tabular}
\end{table}

Here $\mathbf{f}_{ij}$ is the feature vector of edge $e_{ij}$ and $\femb{i}$ is the embedding of node $v_i$. Except for the method of concatenation
(which has a feature vector dimension of $80$) other methods use $40$-dimensional vectors. Since the datasets are typically imbalanced we use the macro-F1 scores to evaluate our method. We repeat this process five times and report the average results (see
Table~\ref{tab:edge}). Some key observations from this table are
as follows: 
\begin{enumerate}[leftmargin=0cm,itemindent=0.5cm,labelwidth=\itemindent,labelsep=0cm,align=left,nosep]
\item \seine{}, not surprisingly, outperforms node2vec across all
datasets. For datasets that contain relatively
fewer negative edges (e.g., $14\%$ for Epinions and $22\%$ for
Slashdot), the improvements are modest (around $34$--$40\%$). For ADJNet and SCOTUS where the sign distribution is less skewed,
\seine{} outperforms node2vec by a huge margin ($64\%$ for ADJNet and $63\%$ for SCOTUS). Also for Wiki the gains are huge (around $66\%$) where $25\%$ of edges are negative.
\item \seine{} demonstrates a consistent advantage over
\sine{} and SNE, with gains ranging from 64--75$\%$ (for the social
network datasets) to 82--115$\%$ (for ADJNet and SCOTUS).
\item \seine{} also outperforms \seine{}-NS in almost all scenarios demonstrating the effectiveness of targeted node sampling over negative sampling.
\item Performance measures (across all scores and across all
algorithms) are comparatively better for Epinions over other datasets
because 
almost $83\%$ of the nodes in Epinions satisfy the structural balance 
condition~\cite{signbal}. As a result edge label prediction is comparatively
easier than in other datasets.
\item The feature scoring method has a noticeable impact w.r.t.
different
datasets. The Average and Concatenation methods subsidize differences whereas
the Hadamard, $\mathcal{L}$-$1$ and $\mathcal{L}$-$2$ methods
promote differences. To understand why this makes a difference, consider
networks like ADJNet and SCOTUS
where connected components denote strong polarities (e.g.,
denoting synonyms or justice leanings, respectively).
In such networks,
the Hadamard, $\mathcal{L}$-$1$ and $\mathcal{L}$-$2$ methods provide
more discriminatory features. However,
Epinions and Slashdot are relatively large datasets with
diversified communities and so
all these methods perform nearly comparably.
\end{enumerate}

\subsubsection{Are representations learned by \seine{} effective at node label prediction?}\label{sec:node}
For datasets like SCOTUS and ADJNet
(where nodes are annotated with labels),
we learn a logistic regression classifier to map
from node representations to corresponding labels (with a 50-50 training-test
split). We also repeat this five times and report the average. See Table~\ref{tab:node} for results.
As can be seen,
\seine{} consistently outperforms all the other
approaches.
In particular,
in the case of SCOTUS which is a citation network, some cases have a
huge number of citations (i.e. landmark cases) in both ideologies. 
Targeted node sampling, by adding such cases to either $\samp{i}$ or $\samn{i}$, situates the
embedding space close to the landmark cases if they are in $\samp{i}$ or away from them if they are in $\samn{i}$, thus supporting accurate node prediction.

\begin{table}[!t]
\centering
\small
\caption{\textit{Comparison of methods for node label prediction on real world datasets. 
\seine{} outperforms other methods in all datasets.}}
\label{tab:node}
\begin{tabular}{l|l|l|l}
\hline
\multicolumn{2}{l|}{Dataset Name}            & ADjNet  & SCOTUS  \\ \hline
\multirow{5}{*}{micro f1}   & node2vec       & 0.5284  & 0.5392  \\
                            & SNE            & 0.5480  & 0.5432  \\
                            & SiNE           & 0.6257  & 0.6131  \\
                            & \seine{}-NS  & 0.7292  & 0.8004  \\
                            & \seine{}     & \textbf{0.8380}  & \textbf{0.8419}  \\ \hline
\multicolumn{2}{l|}{gain over node2vec (\%)}      & 58.5920 & 56.1387 \\
\multicolumn{2}{l|}{gain over SNE (\%)}           & 52.9197 & 54.9890 \\
\multicolumn{2}{l|}{gain over SiNE (\%)}          & 33.9300 & 37.3185 \\
\multicolumn{2}{l|}{gain over \seine{}-NS (\%)} & 14.9205 & 5.1849  \\ \hline
\multirow{5}{*}{macro f1}   & node2vec       & 0.4605  & 0.4922  \\
                            & SNE            & 0.4540   & 0.4435  \\
                            & SiNE           & 0.5847  & 0.5696  \\
                            & \seine{}-NS  & 0.7261  & 0.7997  \\
                            & \seine{}     & \textbf{0.8374}  & \textbf{0.8415}  \\ \hline
\multicolumn{2}{l|}{gain over node2vec (\%)}      & 45.0084 & 41.5092 \\
\multicolumn{2}{l|}{gain over SNE (\%)}           & 84.4493 & 89.7407 \\
\multicolumn{2}{l|}{gain over SiNE (\%)}          & 43.2187 & 47.7353 \\
\multicolumn{2}{l|}{gain over \seine{}-NS (\%)} & 15.3285 & 5.2270  \\ \hline
\end{tabular}
\end{table}

\begin{figure*}[!htbp]
  \centering
  \begin{tabular}{@{}c@{}c@{}c@{}c@{}}
  \includegraphics[width=0.25\textwidth]{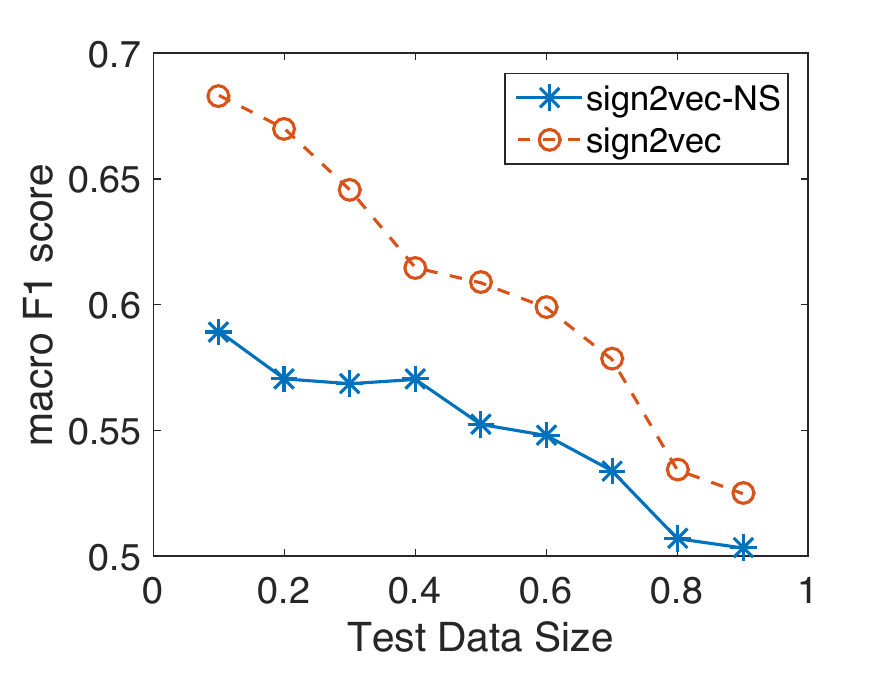} & 
  \includegraphics[width=0.25\textwidth]{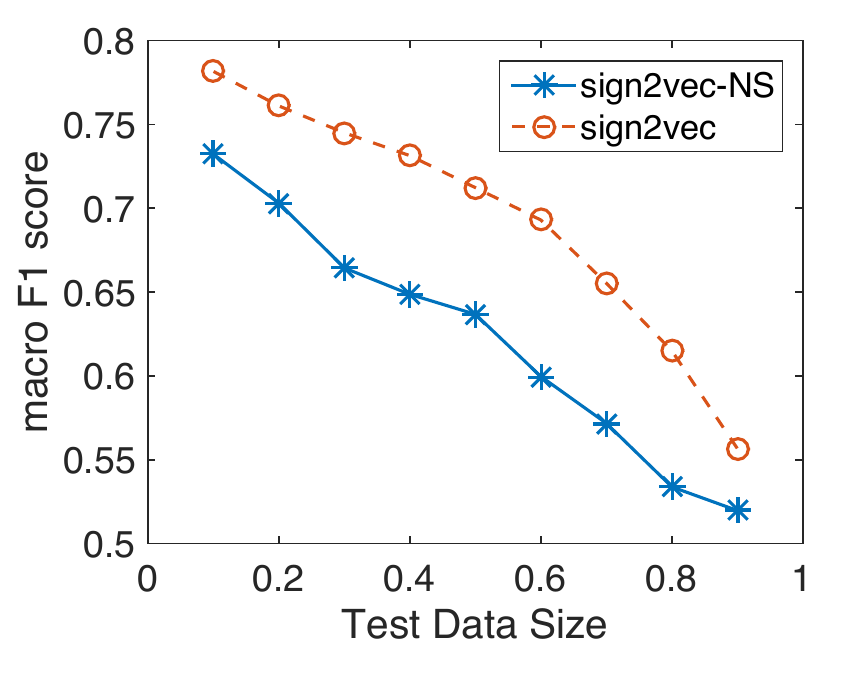} &
  \includegraphics[width=0.25\textwidth]{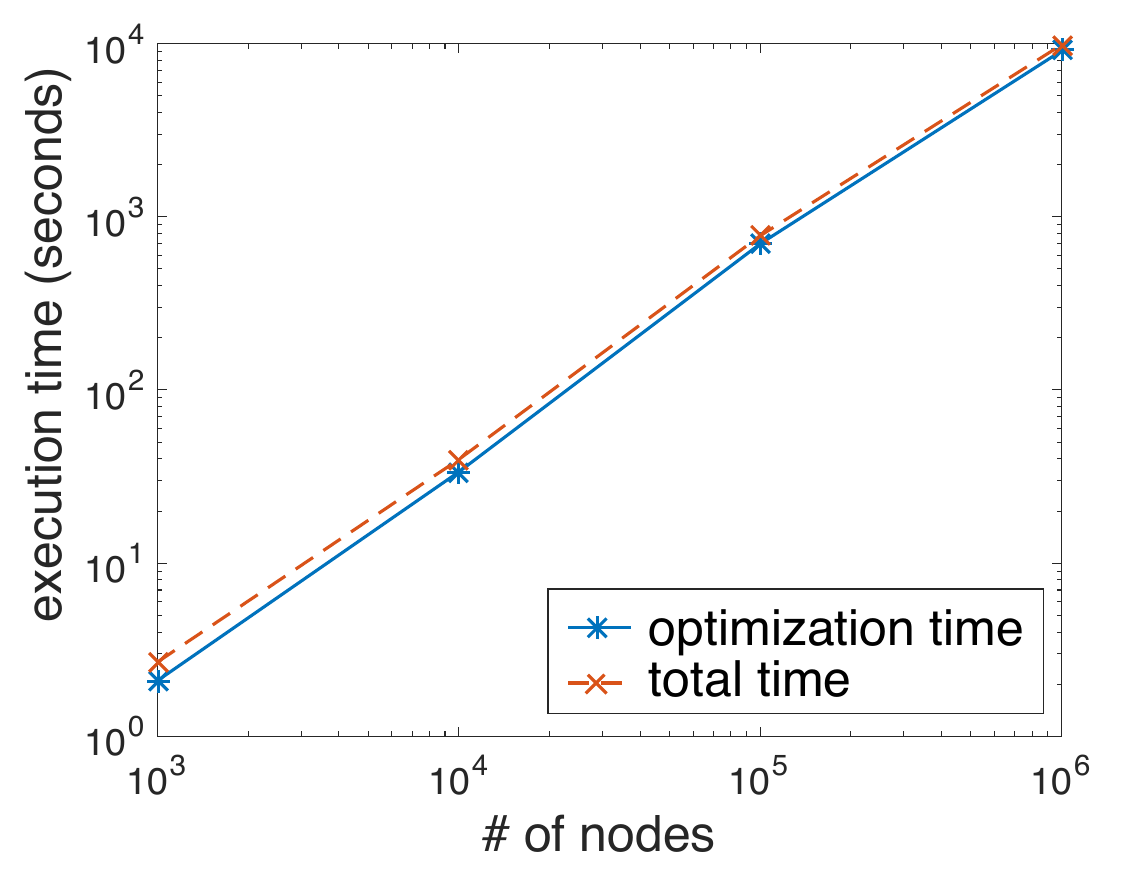} & 
  \includegraphics[width=0.25\textwidth]{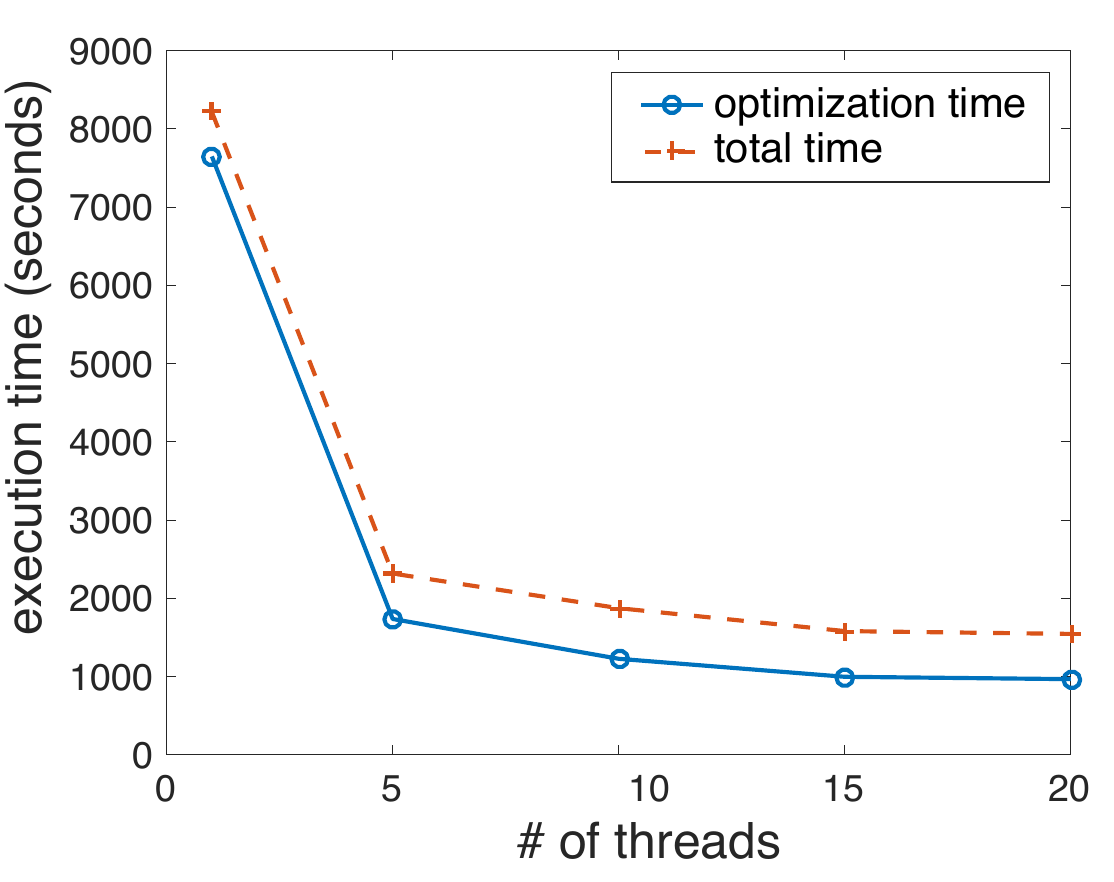} \\
   (a) & (b)  & (c) & (d)  \\
  \end{tabular}
  \caption{\textit{Macro F1 of ADJNet (a) and SCOTUS (b) datasets varying the percent of nodes used for training.  \seine{} outperforms \seine{}-NS in all cases. (c) and (d) show execution time of \seine{} varying the number of nodes and threads.}}
\label{fig:comp}
\end{figure*}
\begin{figure}[!b]
\centering
\includegraphics[width=0.48\textwidth]{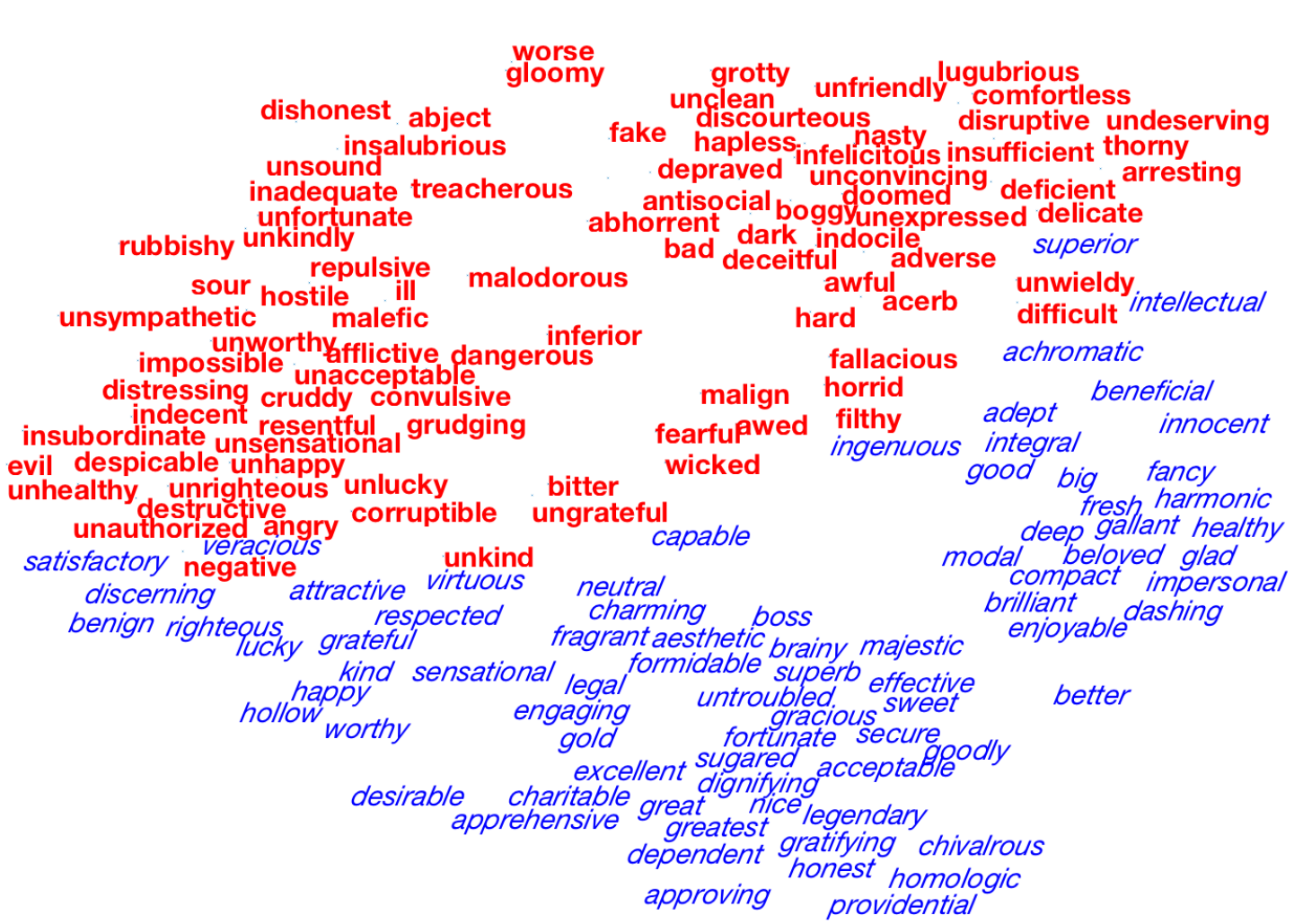}
\caption{Predicting the
polarity of adjectives in a subset of the
ADJNet dataset. Here red labeled/boldface words are negative while the 
blue labeled/slanted words are positive. (Many adjectives have been
removed to reduce clutter.) We use t-SNE to map the data into a 2D space.}
\label{fig:pa}
\end{figure}
The case of
Citizens United vs. Federal Election Commission (FEC),
one of the most controversial cases in recent times,
is instructive. In this case, Citizens United seeks an injection against the FEC to prevent the application of the
Bipartisan Campaign Reform Act (BCRA) so that a film on Hillary Clinton can be broadcasted. In a $5$-$4$ vote, the court decides in favor of Citizens United. 
In Fig.~\ref{fig:sc}, we depict
the BCRA related cases that cite Citizens United vs. Federal Election Commission in a $2$D projection.
The cases whose decisions support a conservative view are shown in red and 
the cases which support a liberal point of view are shown in blue.
Another two cases disputing the application of BCRA cite this case (shown in filled circles), viz.
Williams-Yulee vs The Florida Bar and McCutcheon vs FEC. In the first case the
court supports the liberal point-of-view (shown in blue) and cites the case negatively (shown in dashed line). Therefore, its embedding resides far away from the Citizens United case. In McCutcheon vs FEC, the court supports a
conservative point-of-view and decides in favor of McCutcheon. This case 
positively cites Citizens United case and its embedding is 
therefore positioned closer to it.

\subsubsection{Multiclass Node Classification} \label{sec:syn}
In section \ref{sec:node}, we show the results of node classification on real world dataset. 
One limitation of ADJNet and SCOTUS is nodes are tagged with binary data. Although binary labeling seems plausible in perfectly balanced signed network, it is possible to find the extension of this behavior in many social media analysis. For example, in an election media campaign, there could be multiple candidates, where supporters of one candidate speaks favorably for her candidate while speaks against other candidates. It is interesting to investigate how \seine{} performs in this circumstance.

Unfortunately, to the best of our knowledge there is no publicly available dataset to for this evaluation. That is why we use synthetic dataset to compare the performance.  We generate the networks based on the method proposed in \cite{bal-cut}. Given a total number of nodes $NV$, number of node labels $NG$ and sparsity score $\alpha$, we first create $NG$ subgraphs from $NV$ nodes having only positive edges within the subgraphs. The nodes of $i^{th}$ subgraphs are labeled as class $i$. Then we connect the subgraphs by only by negative edges. We also add random positive and negative edges as noise to make the networks more realistic. $\alpha$ controls the total number of edges. We create $3$ synthetic datasets each with $NV=50000$ nodes where $NG$ is set to $10$ (Syn $10$), $20$ (Syn $20$), $50$ (Syn $50$).

\begin{table}[!t]
\centering
\caption{Comparison of multiclass prediction on Synthetic Datasets. We apply one-vs-rest logistic regression classifier for the prediction. \seine{} outperforms all the other methods in all datasets.}
\label{table:syn}
\begin{tabular}{c|c|c|c|c}
\toprule
\begin{tabular}[c]{@{}c@{}}Performance \\ measure\end{tabular} & Algorithms    & Syn 10  & Syn 20  & Syn 50  \\ \midrule
\multirow{4}{*}{micro f1}                                      & node2vec      & 0.1112  & 0.0527  & 0.0195  \\
                                                              & SiNE          & 0.1105  & 0.0545  & 0.0197  \\
                                                              & \seine{}-NS & 0.1483  & 0.0848  & 0.0519  \\
                                                              & \seine{}    & \textbf{0.1723}  & \textbf{0.1104}  & \textbf{0.0716}  \\ \midrule
\multicolumn{2}{c|}{gain (\%) of \seine{}}                                       & 16.1834 & 30.1887 & 37.9576 \\ \midrule
\multirow{4}{*}{macro f1}                                      & node2vec      & 0.0967  & 0.0283  & 0.0032  \\
                                                              & SiNE          & 0.1083  & 0.0535  & 0.0187  \\
                                                              & \seine{}-NS & 0.1344  & 0.0747  & 0.0486  \\
                                                              & \seine{}    & \textbf{0.1695}  & \textbf{0.1084}  & \textbf{0.0704}  \\ \midrule
\multicolumn{2}{c|}{gain (\%) of \seine{}}                                       & 26.1161 & 45.1138 & 44.8560 \\ \bottomrule
\end{tabular}
\end{table}

We train a one-vs-rest logistic regression classifier for the prediction with a 50-50 training-test split. The result is shown in Table \ref{table:syn}. We can see that, \seine{} not surprisingly outperforms other methods with considerable margin. One of the interesting points is since in this dataset multiple oppositive  groups are present, considering this densely group behavior can provide better node sampling than random walk. This intend to explore this idea in the future.

\subsubsection{How much more effective is our sampling strategy in the presence of partial information?}\label{sec:partial-info}
To evaluate the effectiveness of our targeted node sampling versus negative sampling,
we remove all outgoing edges of a certain percent of randomly selected nodes 
(test nodes), learn an embedding, and then
aim to predict the labels of the test nodes. We show the macro F1 scores 
for ADJNet (treating it as directed) and SCOTUS in Fig.~\ref{fig:comp} (a) and Fig.~\ref{fig:comp} (b). As seen here,
\seine{} consistently
outperforms \seine{}-NS. 
Withholding the outgoing edges of test nodes implies
that both methods will miss the same edge information in
learning the embedding. However due to targeted node sampling many of these test nodes will be added to $\samp{i}$ or $\samn{i}$ in \seine{} (recall only the outgoing edges are removed, but not incoming edges). Because of this property,
\seine{} will be able to make an informed choice while optimizing the embedding space. 
\begin{figure}[!b]
\centering
\includegraphics[scale=0.4]{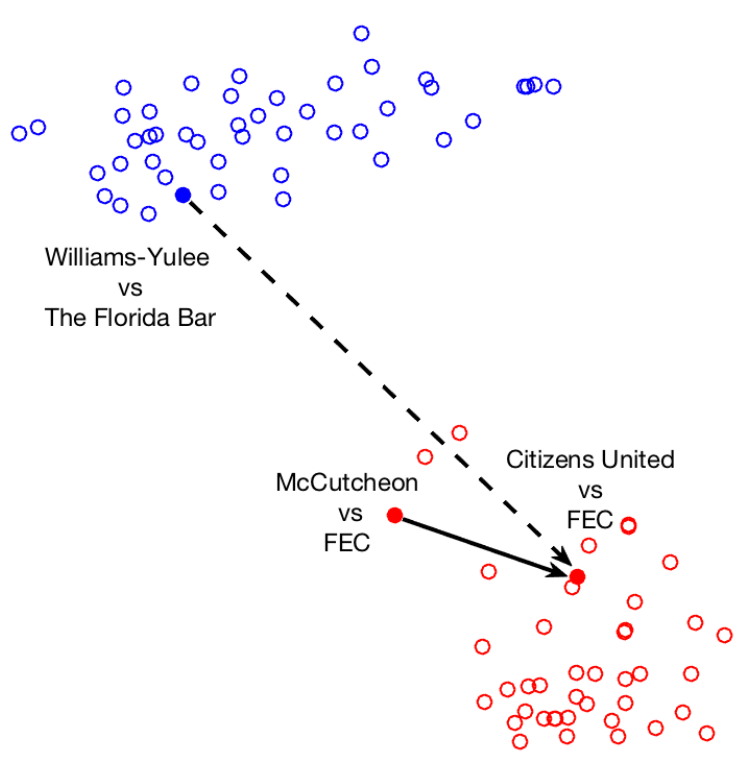}
\caption{\textit{Several conservatively and liberally disputed cases including Bipartisan Campaign Reform Act (BCRA) related cases that cite Citizens United vs. Federal Election Commission. Conservatively (liberally) disputed cases are shown in red (blue). Our discussed cases are shown in filled circles while other cases are shown in unfilled circles. Solid (dashed) edges represent positively (negatively) oriented relationships.}}
\label{fig:sc}
\end{figure}
\subsubsection{How scalable is \seine{} for large networks?}\label{sec:scale}
To assess the scalability of \seine{}, we learn embeddings for an Erdos-Renyi random network for upto one million nodes. The average degree for each node is set to 10 and the total number of samples is set to $100$ times the number of edges in the network. The size of the dimension is also set to 100 for this experiment. We make the network signed by randomly changing the sign of $20\%$ edges to negative. The optimization time and the total execution time (targeted node sampling + optimization) is compared in Fig. \ref{fig:comp} (c) for different vertex sizes. On a regular desktop, an unparallelized version of \seine{}
requires less than 3 hours to learn the embedding space for over 1 
million nodes. Moreover, the sampling times is negligible
compared to the optimization time (less than 15 minutes for 1 million nodes). This actually shows \seine{} is very scalable for real world networks.
Additionally, \seine{} uses an asynchronous stochastic gradient approach, so
it is trivially parallelizable and as Fig.~\ref{fig:comp}(d) shows,
we can obtain a $3.5$ fold improvement with just 5 threads, with diminishing
returns beyond that point.


%% file: sections/reference.tex
\vspace{-1mm}
\section{Other Related Work}
\vspace{-1mm}
Work related to unsupervised feature learning for networks
have been discussed in the introduction. These ideas follow
the trend opened up
originally by
unsupervised feature learning in text. 
Skip-gram models proposed in~\cite{w2v1,w2v2,w2v3} learn a vector representation of words by optimizing a likelihood function. Skip-gram models are based on the principle that words in similar contexts generally have similar meanings~\cite{distStr} and
can be extended to learn feature representations for documents~\cite{doc2vec}, 
parts of speech \cite{sense2vec}, items in
collaborative filtering \cite{item2vec}.
Recently deep learning based models have been proposed for representation learning on graphs to perform the above mentioned prediction tasks in unsigned networks~\cite{lrbm,gateGraph,deepLink,stDeepEmb}. Although these models provide high accuracy by optimizing several layers of non-linear transformations, they are computationally expensive, requires a significant amount of training time and are only applicable to unsigned networks as opposed to our proposed method~\seine{}.  
\vspace{-1mm}

%% file: sections/conc.tex
\vspace{-2mm}
\section{Conclusion}
\vspace{-1mm}
We have presented a scalable feature learning framework suitable
for signed networks. Using a targeted
node sampling for random walks, and leveraging structural
balance theory, we have shown how the embedding space learned
by \seine{} yields interpretable as well as effective representations.
Future work is aimed at experimenting with other theories of signed networks and extensions to networks with a heterogeneity of node
and edge tables.